\documentclass[11pt]{article}
\usepackage[margin=1in]{geometry}
\usepackage{amsmath,amssymb,amsthm}
\usepackage{booktabs}
\usepackage{hyperref}
\usepackage[numbers]{natbib}
\usepackage{xcolor}
\usepackage{algorithm}
\usepackage{algpseudocode}
\usepackage{graphicx}
\usepackage{adjustbox}
\usepackage{tikz}
\usetikzlibrary{arrows.meta,positioning,fit}

\newtheorem{theorem}{Theorem}
\newtheorem{proposition}{Proposition}
\newtheorem{lemma}{Lemma}
\theoremstyle{remark}
\newtheorem{remark}{Remark}

\newcommand{\sgn}{\operatorname{sign}}
\newcommand{\E}{\mathbb{E}}

\title{Ordinal Credit Assignment:\\Multi-Step Reinforcement Learning from Comparisons Alone}
\author{Hao Yu}
\date{Draft, \today}

\begin{document}
\maketitle

\begin{abstract}
Policy-gradient methods need a scalar reward, yet for many tasks no formula captures what we want, while judging which of two outcomes is better is easy. We study reinforcement learning that receives only such comparisons --- no reward function, no reward model --- in multi-step tasks. Groups of trajectories start from a shared state, and each action is credited by how its reward-to-go compares with those of the other trajectories at the same time step. We show that the expected update is policy gradient with the reward-to-go replaced by its win rate against the population still running at that time, weighted by the chance that an opponent is still running; in deterministic environments the update ranks actions in the order of their true values, the optimal policy is stable, and improvable policies are not, whereas under stochastic outcomes it prefers actions that win more often rather than those with higher mean. Label noise concentrated near ties or caused by position bias leaves the expected update unchanged; only noise that also flips clear comparisons halves the signal. A critic trained from the same comparisons removes the state confound of per-step comparisons without changing the expected update. In a pre-registered study over six environments with ten seeds per claim, comparisons cost no detectable performance relative to scalar returns at the same granularity, learning a reward model from the same comparisons is not better and is worse when per-step rewards explain the comparisons poorly, and ten percent of the comparisons suffice when they are spent on a few time points with all pairs. With real LLM judges, comparisons are 7--10 points more accurate than pointwise scores, and the gap grows with the stronger judge. Trained end to end with an LLM judge on a requirement that the hand-written reward omits, policies satisfy the requirement in all five seeds, but land far less often than with oracle comparisons, because the judge is near chance on the near-tied failures that dominate early training. Comparisons do not close the gap to PPO with scalar rewards on HalfCheetah.
\end{abstract}

\section{Introduction}
\label{sec:intro}

Policy-gradient methods need a scalar reward, but for many tasks no formula captures every property we want, and a hand-written reward introduces bias that the policy will happily exploit. Judging which of two outcomes is better is often easier and more reliable than assigning either a score; this is the premise of learning from human preferences~\citep{christiano2017deep,ouyang2022instructgpt}, usually through a learned reward model or offline objectives such as DPO~\citep{rafailov2023dpo}. The question we study is what reinforcement learning can do with that weaker signal alone: no reward function, no reward model, only comparisons.

Using within-group win rates as the advantage in GRPO-style updates is not new; it has been proposed for text-to-image models~\citep{wang2025prefgrpo}, LLM agents~\citep{zhang2026arenarl,yang2026tournament}, rank-only advantages~\citep{choi2026gopo}, and, at the trajectory level, for continuous control~\citep{swamy2024minimaximalist}. The same estimator is the centered-rank fitness shaping of evolution strategies~\citep{wierstra2014nes,salimans2017es}. What has not been worked out is how to use comparisons for \emph{per-step} credit in multi-step tasks, what such updates actually optimize, and how they behave under the label noise real judges produce.

\paragraph{Contributions.}
\begin{enumerate}
\item \textbf{What the ordinal update optimizes, and which noise hurts it} (Section~\ref{sec:theory}). In multi-step tasks the expected update is policy gradient with the reward-to-go replaced by its win rate against the trajectories still running at the same time step (Theorem~\ref{thm:multistep}); in deterministic environments it ranks actions by their true values, keeps the optimal policy stable and destabilizes improvable ones, while under stochastic outcomes it can prefer a lower-mean action. Symmetric label noise that is concentrated near ties, and position bias, leave the expected update direction intact; only noise that also flips clear comparisons halves the signal. This predicts an experimental finding: at the same 25\% flip rate, independent flips collapse the method while near-tie and positional noise barely affect it.
\item \textbf{A comparison-trained critic as a control variate} (Section~\ref{sec:method}). Comparing reward-to-go across trajectories at the same time step mixes the value of the state with the quality of the action. We train a state ranker from the same comparisons with a Bradley--Terry loss and subtract the comparison it predicts from the states alone. Because the correction depends only on states and on an independent trajectory, the expected update is unchanged. Two variants that let a learned model supply the advantage fail, which yields a design rule: the advantage must keep the observed comparison.
\item \textbf{A pre-registered empirical study} (Section~\ref{sec:experiments}). Every main claim has its criteria written before the runs, 10--15 seeds, and multiple-comparison correction; we report the negative results. Using comparisons instead of scalar returns has no detectable cost at the same granularity; with real LLM judges, comparisons are 7--10 points more accurate than pointwise scores, and the gap widens with the stronger judge; learning a reward model from the same comparisons is not better and is worse when per-step rewards explain the comparisons poorly; and 10\% of the comparisons suffice when they are spent on a few time points with all pairs. In a five-seed pilot, training on LLM-judge comparisons alone transfers a natural-language requirement that the hand-written reward lacks, but with a much lower landing rate than oracle comparisons. Comparisons do not close the gap to PPO on HalfCheetah, and they are not uniformly less sensitive to hyperparameters.
\end{enumerate}

\section{Related work}
\label{sec:related}

\paragraph{Win-rate and rank advantages.}
Centered ranks were introduced as fitness shaping in natural evolution strategies~\citep{wierstra2014nes,salimans2017es}. In policy optimization for language and image models, within-group pairwise win rates or rank-only advantages replace scalar rewards in variants of GRPO~\citep{shao2024grpo} and RLOO-style group baselines~\citep{ahmadian2024rloo}~\citep{wang2025prefgrpo,choi2026gopo,zhang2026arenarl,yang2026tournament}. These operate on whole responses or whole trajectories. Self-play formulations~\citep{munos2023nash,wu2024sppo,swamy2024minimaximalist} make the win-rate objective explicit; \citet{dudik2015contextual} and \citet{zhang2025winrate} relate it to reward maximization. We use these results as lemmas and study the multi-step, per-step setting they leave open.

\paragraph{Preference-based RL.}
The classic route learns a reward model from trajectory or segment preferences and then runs RL~\citep{christiano2017deep,lee2021pebble}, including with vision-language models as the source of preferences~\citep{wang2024rlvlmf}; learned reward models can be over-optimized~\citep{gao2023overoptimization}. Reward-model-free alternatives learn advantages or Q-functions directly from preferences offline~\citep{knox2024advantage,hejna2024cpl,hejna2023ipl}, or estimate policy-value differences online~\citep{zhang2025zpg,assmus2026sp3o}. We compare against the reward-model route with exactly the same comparisons, and against trajectory-level self-play~\citep{swamy2024minimaximalist}.

\paragraph{Per-step credit without a critic.}
Comparing reward-to-go across rollouts at the same time index goes back to PI$^2$~\citep{theodorou2010pi2}; anchor-state grouping~\citep{feng2025gigpo} and Monte Carlo rollouts from intermediate states~\citep{kazemnejad2025vineppo} assign per-step credit with scalar rewards. Our setting removes the scalar.

\paragraph{Ties and judge noise.}
Ties have been modeled in offline preference optimization~\citep{chen2025dpoties}. LLM judges show strong position bias~\citep{yagubyan2026coinflip}; our noise analysis covers that case.

\begin{table}[t]
\centering
\caption{Closest prior work. ``Signal'': what the learner observes. ``Credit'': granularity at which credit is assigned. RM: a reward model is fitted.}
\label{tab:related}
\small
\adjustbox{max width=\linewidth}{\begin{tabular}{lllcc}
\toprule
method & signal & credit & RM & online \\
\midrule
Christiano et al.\ / PEBBLE~\citep{christiano2017deep,lee2021pebble} & segment preferences & per step, via $\hat r$ & yes & yes \\
CPL / IPL~\citep{hejna2024cpl,hejna2023ipl} & segment preferences & per step (advantage/$Q$) & no & no \\
ZPG / SP3O~\citep{zhang2025zpg,assmus2026sp3o} & trajectory/segment preferences & policy-value difference & no & yes \\
SPO~\citep{swamy2024minimaximalist} & trajectory preferences & trajectory & no & yes \\
Pref-GRPO, GOPO, ArenaRL, Tournament-GRPO~\citep{wang2025prefgrpo,choi2026gopo,zhang2026arenarl,yang2026tournament} & pairwise / rank in group & response or trajectory & no & yes \\
GiGPO, GR2PO, PI$^2$~\citep{feng2025gigpo,theodorou2010pi2} & scalar reward & per step (grouped) & no & yes \\
\midrule
this work & pairwise, ties allowed & per step (+ comparison critic) & no & yes \\
\bottomrule
\end{tabular}}
\end{table}

\section{Method}
\label{sec:method}

\begin{figure}[t]
\centering
\begin{tikzpicture}[font=\small, >=Stealth, node distance=4mm]
  \node[circle, draw, inner sep=1.5pt] (s0) at (0,0) {$s_0$};
  \foreach \y/\n in {1.2/1, 0.4/2, -0.4/3, -1.2/K} {
    \draw[->, thick, gray!70] (s0) .. controls (0.8,\y) .. (6.2,\y) node[right, black] {$\tau^{\n}$};
  }
  \foreach \x/\lab in {1.8/{t_1}, 3.6/{t_2}, 5.4/{t_3}} {
    \draw[dashed, gray] (\x,1.55) -- (\x,-1.55);
    \node[below] at (\x,-1.55) {$\lab$};
    \foreach \y in {1.2, 0.4, -0.4, -1.2} { \fill (\x,\y) circle (1.3pt); }
  }
  \node[draw, rounded corners, align=center, fill=gray!8] (judge) at (9.3,0.75)
    {comparisons at $t$\\[1pt] $c^{ij}_t\in\{-1,0,+1\}$, all pairs\\[1pt]\footnotesize (oracle, human or LLM judge)};
  \node[draw, rounded corners, align=center, fill=gray!8] (adv) at (9.3,-0.95)
    {$\tilde A^i_t=\mathrm{mean}_j\big[c^{ij}_t-\hat c(s^i_t,s^j_t)\big]$\\[1pt]\footnotesize $\hat c$: critic trained on the same labels};
  \draw[->] (6.9,0.75) -- (judge.west);
  \draw[->] (judge) -- (adv);
  \draw[->] (adv.south) -- ++(0,-0.45) node[below] {PPO-clip update; actions in a window share the advantage of its start};
\end{tikzpicture}
\caption{Per-step comparison policy optimization. $K$ trajectories start from a shared state. At a few time points, every pair of still-running trajectories is compared on what happened from that point on; no scalar reward is observed. Each action's advantage is its mean comparison outcome minus the outcome a comparison-trained critic predicts from the two states alone.}
\label{fig:method}
\end{figure}

\subsection{Setting}
An agent interacts with an episodic MDP with policy $\pi_\theta$. Each update samples groups of $K$ trajectories $\tau^1,\dots,\tau^K$ that start from the same initial state and then evolve independently. For trajectory $i$ let $T_i$ be its length and $G^i_t=\sum_{k=t}^{T_i-1}\gamma^{k-t}r^i_k$ its reward-to-go, defined for $t<T_i$. The algorithm never observes $r$ or $G$. It observes only comparison labels
\begin{equation}
c^{ij}_t \in \{-1,0,+1\}, \qquad i,j\in A_t,
\end{equation}
where $A_t=\{i:\,t<T_i\}$ is the set of trajectories still running at time $t$, and $c^{ij}_t=+1$ means ``from time $t$ on, $i$ did better than $j$''. Noiseless labels are $c^{ij}_t=\sgn(G^i_t-G^j_t)$; Section~\ref{sec:noise} treats noisy ones. Labels are antisymmetric, $c^{ji}_t=-c^{ij}_t$.

\subsection{Per-step comparison advantage}
The per-step comparison advantage of the action $a^i_t$ is the mean comparison against the other running trajectories,
\begin{equation}
\label{eq:pair}
\hat A^i_t = \frac{1}{|A_t|-1}\sum_{j\in A_t\setminus\{i\}} c^{ij}_t,
\end{equation}
set to zero when $|A_t|<2$. It is plugged into the clipped surrogate of PPO~\citep{schulman2017ppo}; advantages are divided by their batch standard deviation (not re-centered, since \eqref{eq:pair} already sums to zero within each $(\text{group},t)$).

\subsection{Comparison-trained critic}
Trajectories compared at the same $t$ sit in different states, so $c^{ij}_t$ reflects both how good $s^i_t$ is and how good the subsequent actions were. We learn a state ranker $V_\phi$ from the same labels with a Bradley--Terry loss~\citep{bradley1952rank},
\begin{equation}
\label{eq:bt}
\begin{split}
\mathcal{L}(\phi) = -\sum_{(i,j,t)} \Big[ & y^{ij}_t \log\sigma\big(V_\phi(s^i_t)-V_\phi(s^j_t)\big) \\
 & + (1-y^{ij}_t)\log\sigma\big(V_\phi(s^j_t)-V_\phi(s^i_t)\big)\Big],\qquad y^{ij}_t=\tfrac{1+c^{ij}_t}{2},
\end{split}
\end{equation}
so a tie has target $\tfrac12$. The advantage subtracts the comparison predicted from the two states alone:
\begin{equation}
\label{eq:res}
\tilde A^i_t = \frac{1}{|A_t|-1}\sum_{j\in A_t\setminus\{i\}} \Big[ c^{ij}_t - \big(2\sigma(V_\phi(s^i_t)-V_\phi(s^j_t))-1\big) \Big].
\end{equation}
\textbf{Order of operations matters:} $\tilde A$ for the current batch is computed with the critic from earlier batches, and only afterwards is $V_\phi$ trained on the current batch's labels. This is what makes Proposition~\ref{prop:cv} hold.

\begin{algorithm}[t]
\caption{Per-step comparison policy optimization with a comparison-trained critic}
\label{alg:main}
\begin{algorithmic}[1]
\For{each update}
  \State Sample groups of $K$ trajectories, each group from a shared initial state
  \State Query labels $c^{ij}_t$ at the chosen time points (all pairs per time point)
  \State Compute $\tilde A^i_t$ by \eqref{eq:res} with the current $V_\phi$ \Comment{$\hat A$ by \eqref{eq:pair} without a critic}
  \State Update $\pi_\theta$ with PPO-clip on $\tilde A / \operatorname{std}(\tilde A)$
  \State Update $V_\phi$ on this batch's labels with \eqref{eq:bt}
\EndFor
\end{algorithmic}
\end{algorithm}

\subsection{Comparison budget}
All pairs at every step costs $O(K^2T)$ labels per group. Section~\ref{sec:exp-budget} shows that comparing all pairs at every $L$-th step, and letting the actions in each window share the advantage of the window start, keeps performance at $10\%$ of the labels, whereas comparing a single random pair at every step does not.

\section{What the comparison update optimizes}
\label{sec:theory}

\subsection{Rank structure and invariance}
\begin{proposition}[Centered rank and monotone invariance]
\label{prop:rank}
With noiseless labels and all pairs compared, $\hat A^i_t = \frac{2\rho^i_t-(n-1)}{n-1}$, where $n=|A_t|$ and $\rho^i_t$ is the number of running trajectories with smaller reward-to-go at $t$, ties counted one half. Consequently $\hat A$ is unchanged if, at every $t$, the compared quantity $G_t$ is replaced by $f_t(G_t)$ for any strictly increasing $f_t$.
\end{proposition}
\begin{proof}
$\sum_{j\ne i}\sgn(G^i_t-G^j_t) = \#\{j:G^j_t<G^i_t\} - \#\{j:G^j_t>G^i_t\} = \rho^i_t-\big((n-1)-\rho^i_t\big)$, with ties contributing $0 = \tfrac12-\tfrac12$. Strictly increasing maps preserve every sign.
\end{proof}
\begin{remark}
The invariance concerns the compared outcome, not the per-step reward: a monotone transform of $r$ need not preserve the order of sums. For single-step problems the two coincide. Standard GRPO, which z-scores within the group, is invariant only to group-wise affine maps.
\end{remark}

\subsection{Single-step case: win rate against the current policy}
Let the action space be finite, $\pi$ a policy, and define the win rate of action $a$ against $\pi$ under labels $c(a,b)$,
\begin{equation}
W_\pi(a) = \E_{b\sim\pi}\big[c(a,b)\big].
\end{equation}
For noiseless labels with a random reward $R$, $W_\pi(a)=\Pr(R(a)>R(b))-\Pr(R(a)<R(b))$.

\begin{lemma}[Unbiasedness]
\label{lem:unbiased}
With $K$ i.i.d.\ samples from $\pi_\theta$, $\E\big[\hat A^i\,\big|\,a^i\big]=W_{\pi_\theta}(a^i)$. Hence the expected update $\E[\nabla_\theta\log\pi_\theta(a^i)\hat A^i]$ equals $\nabla_{\theta'}\E_{a\sim\pi_{\theta'}}[W_{\pi_\theta}(a)]\big|_{\theta'=\theta}$, the gradient of the win rate against a frozen copy of the current policy.
\end{lemma}
\begin{proof}
Given $a^i$, the other $K-1$ samples are independent draws from $\pi_\theta$, so each $c(a^i,a^j)$ has conditional mean $W_{\pi_\theta}(a^i)$. The second statement is the score-function identity~\citep{williams1992reinforce} with $W_{\pi_\theta}$ held fixed.
\end{proof}

\begin{proposition}[Deterministic rewards: only optimal policies are stable]
\label{prop:det}
Let $R$ be deterministic, $\pi$ a softmax over logits $z$, and consider the expected-update (gradient-flow) dynamics $\dot z=\E[\nabla_z\log\pi(a)\hat A(a)]$. Then
\begin{enumerate}
\item[(i)] $\dot z_a=\pi(a)\,W_\pi(a)$, and $\E_{a\sim\pi}W_\pi(a)=0$;
\item[(ii)] $\dot z=0$ if and only if all actions in the support of $\pi$ have the same reward;
\item[(iii)] a drift-free policy whose support has reward below the maximum is unstable: if an $\varepsilon$ of probability is moved to a strictly better action $a^+$, then $\frac{d}{d\tau}\pi(a^+)>0$; a drift-free policy supported on maximizers is stable: if an $\varepsilon$ of probability is moved to worse actions, their logits decrease.
\end{enumerate}
\end{proposition}
\begin{proof}
(i) For softmax, $\partial_{z_a}\log\pi(b)=\mathbb{1}[a=b]-\pi(a)$, so $\dot z_a=\pi(a)W_\pi(a)-\pi(a)\E_\pi W_\pi$, and $\E_{a,b\sim\pi}c(a,b)=0$ by antisymmetry.
(ii) If the support has two reward levels, an action $\bar a$ at the highest supported level has $W_\pi(\bar a)=\Pr_{b\sim\pi}(R(b)<R(\bar a))>0$, so $\dot z_{\bar a}>0$. Conversely, a single supported level gives $W_\pi\equiv 0$ on the support.
(iii) Let $\pi$ be supported on a level set $L$ and move $\varepsilon$ to $a^+$ with $R(a^+)>R(b)$ for $b\in L$. Then $W(a^+)=1-\varepsilon$ and $W(b)=-\varepsilon$ for $b\in L$, so $\dot z_{a^+}=\varepsilon(1-\varepsilon)$, $\dot z_b=-\pi(b)\varepsilon$, and
$\frac{d}{d\tau}\pi(a^+)=\pi(a^+)\big(\dot z_{a^+}-\sum_b\pi(b)\dot z_b\big)=\varepsilon^2\big((1-\varepsilon)^2+\textstyle\sum_{b\in L}\pi(b)^2\big)>0$.
The stable case is symmetric: actions worse than all of the support have $W=-(1-\varepsilon)<0$.
\end{proof}
\begin{remark}
The probability of the set of maximizers is \emph{not} monotone along the dynamics: when a second-best action holds most of the mass and beats the remaining ones, its logit can grow faster than that of the best action. A random search over 200{,}000 reward/policy pairs found this in 3.2\% of cases (e.g.\ rewards $(3,3,1,0,0,0)$, $\pi=(0.016,0.04,0.68,0.078,0.103,0.082)$). The proposition therefore states stability, not monotone improvement.
\end{remark}

\begin{proposition}[Stochastic rewards: the win-rate optimum can have lower mean]
\label{prop:skew}
Let $R(a_1)=0$ and $R(a_2)=1$ with probability $0.9$, $-20$ otherwise ($\E R(a_2)=-1.1$). Then $W(a_2)$ against $a_1$ is $0.8>0$, so the comparison update moves toward $a_2$, whereas expected-reward maximization chooses $a_1$.
\end{proposition}
\begin{remark}
Group-normalized GRPO has the same failure in our bandit runs (20/20 seeds pick the lower-mean arm); only the unnormalized linear advantage picks $a_1$. The bias of group standard-deviation normalization was noted by \citet{liu2025drgrpo}.
\end{remark}

\subsection{Label noise}
\label{sec:noise}
Write $\Delta=G^i_t-G^j_t$ and consider labels that are flipped with a probability depending only on $|\Delta|$, $\tilde c=\sgn(\Delta)\cdot\xi$ with $\Pr(\xi=-1\mid\Delta)=q(|\Delta|)<\tfrac12$, ties kept as ties.

\begin{proposition}[Which noise preserves the expected update]
\label{prop:noise}
$\E[\tilde c\mid\Delta]=g(\Delta):=\sgn(\Delta)\big(1-2q(|\Delta|)\big)$.
\begin{enumerate}
\item[(a)] If $q\equiv\varepsilon$ (independent flips), $g=(1-2\varepsilon)\sgn$: the expected update is scaled by $1-2\varepsilon$ while the per-label variance is not reduced.
\item[(b)] If $q$ is non-increasing in $|\Delta|$ (errors concentrated near ties), $g$ is odd and non-decreasing. Lemma~\ref{lem:unbiased} and Proposition~\ref{prop:det} hold with $c$ replaced by $g$, and comparisons with large $|\Delta|$ keep almost full weight.
\item[(c)] Position bias --- with probability $\beta$ the label is $+1$ for whichever item was presented first, independently of the outcomes --- adds to $\hat A^i$ a term that depends only on the presentation order. If the order does not depend on the actions, this term has zero correlation with $\nabla\log\pi(a^i_t)$, so the expected update is unchanged and only its variance grows.
\end{enumerate}
\end{proposition}
\begin{proof}
(a),(b): direct computation; for (b), $1-2q(|\Delta|)\ge 0$ and non-decreasing in $|\Delta|$, so $g$ is non-decreasing on each half-line, and the proofs of Lemma~\ref{lem:unbiased} and Proposition~\ref{prop:det} use only antisymmetry of the label and its sign for $\Delta\neq 0$, both of which $g$ preserves when $q<\tfrac12$. (c): the added term is a function of the presentation order, which is independent of $a^i_t$ given $s^i_t$, so the score function has zero mean against it.
\end{proof}

\subsection{The comparison-trained critic does not change the expected update}
\begin{proposition}[Control variate]
\label{prop:cv}
Suppose $V_\phi$ does not depend on the current batch. Then for every $(i,t)$,
\[
\E\big[\nabla_\theta\log\pi_\theta(a^i_t\mid s^i_t)\,\tilde A^i_t\big]=\E\big[\nabla_\theta\log\pi_\theta(a^i_t\mid s^i_t)\,\hat A^i_t\big].
\]
\end{proposition}
\begin{proof}
The difference is $\E[\nabla\log\pi(a^i_t\mid s^i_t)\,B]$ with $B=\frac{1}{|A_t|-1}\sum_{j}\big(2\sigma(V_\phi(s^i_t)-V_\phi(s^j_t))-1\big)$. $B$ depends on $s^i_t$, on the states $s^j_t$ of other trajectories, and on which of them are still running at $t$. Trajectories $j\ne i$ are generated independently of trajectory $i$ after the shared initial state, and whether $i$ is running at $t$ is decided before $a^i_t$. Hence, conditionally on $s^i_t$ and on all other trajectories, $a^i_t\sim\pi_\theta(\cdot\mid s^i_t)$ and $\E[\nabla\log\pi(a^i_t\mid s^i_t)\mid s^i_t,\{\tau^j\}_{j\ne i}]=0$.
\end{proof}
\begin{remark}
The variance reduction is that of a control variate: it is largest when $2\sigma(V(s^i)-V(s^j))-1$ approximates $\E[c^{ij}_t\mid s^i_t,s^j_t]$. If $V_\phi$ were trained on the current batch first, $B$ would depend on $a^i_t$ through the labels and the identity would fail.
\end{remark}

\subsection{The multi-step objective}
\label{sec:multistep}
We now characterize the per-step update \eqref{eq:pair} exactly. Let $T$ be the (random) length of a trajectory, $p_t=\Pr(T>t)$ the probability that it is still running at $t$, and for an independent trajectory $\tau'$ of the current policy from the same initial state define
\begin{equation}
\label{eq:ut}
u_t(x)=\Pr\big(G'_t<x\,\big|\,T'>t\big)-\Pr\big(G'_t>x\,\big|\,T'>t\big),\qquad c_t=1-(1-p_t)^{K-1}.
\end{equation}
$u_t$ is the win rate of a reward-to-go value $x$ against the population of trajectories still running at time $t$; it is non-decreasing in $x$ and takes values in $[-1,1]$. $c_t$ is the probability that at least one of the other $K-1$ trajectories is still running at $t$.

\begin{theorem}[Expected per-step comparison update]
\label{thm:multistep}
With $K$ trajectories that are independent after a shared initial state and noiseless labels, the per-step comparison estimator
\[
g=\sum_{i=1}^{K}\sum_{t<T_i}\nabla_\theta\log\pi_\theta(a^i_t\mid s^i_t)\,\hat A^i_t
\]
satisfies
\[
\E[g]=K\sum_{t}\E\Big[\mathbb{1}\{T>t\}\;c_t\,u_t(G_t)\;\nabla_\theta\log\pi_\theta(a_t\mid s_t)\Big].
\]
\end{theorem}
\begin{proof}
Fix $i$ and $t<T_i$. Let $N$ be the number of other trajectories still running at $t$; $N\sim\mathrm{Bin}(K-1,p_t)$ independently of $\tau^i$. Given $\tau^i$ and $N=n\ge1$, the reward-to-go values of the running opponents are i.i.d.\ draws from the law of $G'_t$ given $T'>t$, again independent of $\tau^i$, so $\E[\hat A^i_t\mid\tau^i,N=n]=u_t(G^i_t)$; for $n=0$, $\hat A^i_t=0$. Hence $\E[\hat A^i_t\mid\tau^i]=\Pr(N\ge1)\,u_t(G^i_t)=c_t\,u_t(G^i_t)$, and the score factor depends only on $\tau^i$. Summing over the $K$ exchangeable trajectories gives the claim.
\end{proof}

\begin{remark}[Reading of the theorem]
Policy gradient with reward-to-go and a time-dependent baseline uses $\E[\sum_t\nabla\log\pi(a_t\mid s_t)(G_t-b_t)]$. The comparison update replaces $G_t$ by $c_t\,u_t(G_t)$: (i) a non-decreasing transform of the reward-to-go, so only its rank within the population at time $t$ matters; (ii) the same transform for every state at time $t$, which is why comparing trajectories that sit in different states is still a valid policy-gradient signal; (iii) a weight $c_t$ that shrinks late time steps reached by few trajectories. With noisy labels of the form of Proposition~\ref{prop:noise}, $u_t(x)$ becomes $\E[g(x-G'_t)\mid T'>t]$, which is still non-decreasing. For a policy with separate parameters $\theta_t$ per time step, the $\theta_t$-component of $\E[g]$ is $K c_t\,\nabla_{\theta_t}\E[\mathbb{1}\{T>t\}u_t(G_t)]$ with $u_t$ held fixed: each time step ascends the probability that its reward-to-go beats that of an independent running trajectory. Unless $u_t$ is affine, the sum of these per-time-step ascents need not be the gradient of a single objective.
\end{remark}

\begin{proposition}[Deterministic environments]
\label{prop:detmdp}
Let transitions and rewards be deterministic and the policy tabular, time-indexed and softmax. Then the expected drift of the logit $z_t(s,a)$ is $K c_t\,d_t(s)\,\pi_t(a\mid s)\big[\bar u_t(s,a)-\sum_b\pi_t(b\mid s)\bar u_t(s,b)\big]$ with $\bar u_t(s,a)=\E[u_t(G_t)\mid s_t=s,a_t=a]$ and $d_t$ the state distribution at $t$, and:
\begin{enumerate}
\item[(a)] \emph{Ordering.} If the policy after time $t$ is deterministic, then at any reachable state $s$, for actions $a,b$ with positive probability, $Q^\pi_t(s,a)>Q^\pi_t(s,b)$ implies $\bar u_t(s,a)>\bar u_t(s,b)$.
\item[(b)] \emph{The optimal policy is stable.} Let $\pi^\star$ be a deterministic optimal policy and move probability $\varepsilon$ at an on-path pair $(t,s)$ to an action $a'$ with $Q^\star_t(s,a')<Q^\star_t(s,\pi^\star(s))$. Then the logit of $a'$ decreases.
\item[(c)] \emph{Improvable policies are unstable.} Let $\pi$ be deterministic and let $a^+$ satisfy $Q^\pi_t(s,a^+)>V^\pi_t(s)$ at an on-path pair $(t,s)$. If probability $\varepsilon$ is moved to $a^+$, its logit increases.
\end{enumerate}
\end{proposition}
\begin{proof}
The drift formula is Theorem~\ref{thm:multistep} with $\partial_{z_t(s,a)}\log\pi_t(b\mid s)=\mathbb{1}[a=b]-\pi_t(a\mid s)$.
(a) With deterministic dynamics and continuation, $G_t=Q^\pi_t(s,a)$ almost surely given $(s_t,a_t)=(s,a)$, so $\bar u_t(s,a)=u_t(Q^\pi_t(s,a))$. For $y>x$, $u_t(y)-u_t(x)=\Pr(G'_t\in[x,y))+\Pr(G'_t\in(x,y])\ge\Pr(G'_t=x)+\Pr(G'_t=y)$, and both point masses are positive because opponents pass through $(s,b)$ and $(s,a)$ with positive probability and then have exactly these values.
(b) Opponents at time $t$ have $G'_t=G^\star_t$ with probability $1-\varepsilon$ and $Q^\star_t(s,a')<G^\star_t$ with probability $\varepsilon$. Hence $u_t(Q^\star_t(s,a'))=-(1-\varepsilon)$, $u_t(G^\star_t)=\varepsilon$, the state average is $0$, and the drift of $z_t(s,a')$ is $-K c_t\,\varepsilon(1-\varepsilon)<0$.
(c) Symmetrically, $u_t(Q^\pi_t(s,a^+))=1-\varepsilon$, the state average is $0$, and the drift of $z_t(s,a^+)$ is $K c_t\,\varepsilon(1-\varepsilon)>0$.
\end{proof}

\begin{remark}[What is not covered]
If a deterministic suboptimal policy has no improving action on its own path, the expected update gives no on-path signal; this is the same local structure as policy gradient near deterministic policies, so comparisons neither help nor hurt there. The deterministic-continuation assumption in (a) is needed: with a stochastic continuation the update ranks actions by $\E[u_t(G_t)]$ rather than by $u_t(\E[G_t])$, the multi-step form of Proposition~\ref{prop:skew}. In 300 random deterministic MDPs with fully stochastic policies, 2.1\% of ordered action pairs (107 of 5{,}011) were ranked opposite to $Q^\pi$. Global convergence with function approximation and stochastic dynamics remains open.
\end{remark}

\begin{remark}[Numerical checks]
Theorem~\ref{thm:multistep} was checked on five random MDPs with stochastic transitions, rewards and early termination ($c_t$ as low as $0.19$, $K\in\{2,\dots,5\}$) by exact enumeration against $4\times10^5$ sampled groups: all components agree within $2.02$ Monte Carlo standard errors, while the same formula without $c_t$ is off by 14--124 standard errors. Proposition~\ref{prop:detmdp}(a)--(c) held in all 1{,}670, 2{,}176 and 277 cases on 300 random deterministic MDPs with deliberately tied integer rewards (\texttt{clean/theory\_check.py}, \texttt{clean/theory\_check2.py}).
\end{remark}

\section{Experiments}
\label{sec:experiments}

\paragraph{Protocol.}
Six environments (CartPole, Acrobot, LunarLander, Hopper, HalfCheetah, MiniGrid DoorKey-5$\times$5). All methods share the actor, PPO-clip settings~\citep{schulman2017ppo} and data collection (PPO additionally learns a critic with GAE~\citep{schulman2016gae}): $K=8$ trajectories per group from a shared initial state, 8192 steps per update. Labels are generated from the environment return; the algorithms see only labels. Every main claim was pre-registered with its test, correction, and decision rule before the runs; tuning seeds are disjoint from evaluation seeds. We report mean $\pm$ 95\% confidence interval over 10 seeds unless stated, Welch $t$-tests with Holm correction within each family, and Mann--Whitney tests as a check. Code, raw results, and every pre-registration file (written before the corresponding runs) are in the supplementary material; Appendix~\ref{app:prereg} summarizes them and Appendix~\ref{app:deviations} lists every deviation and bug we found. Figure~\ref{fig:curves} shows learning curves.

\begin{figure}[t]
\centering
\includegraphics[width=\linewidth]{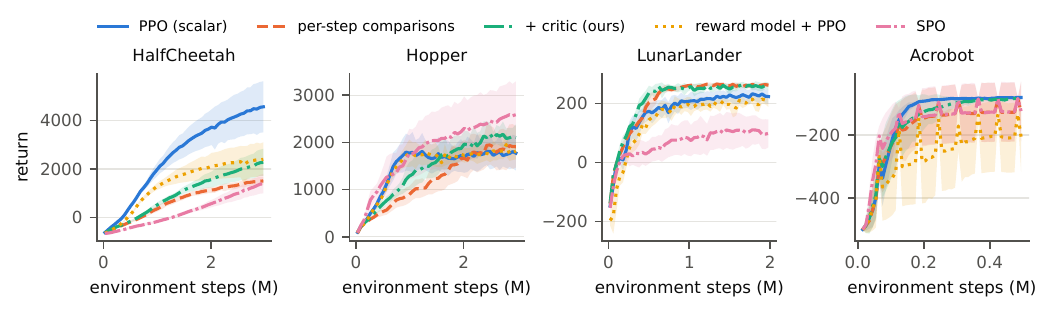}
\caption{Learning curves (mean and 95\% confidence band over 10 seeds; HalfCheetah and Hopper 3M steps, LunarLander 2M, Acrobot 0.5M). PPO uses scalar rewards; the other methods see only comparisons. Final values are in Tables~\ref{tab:critic} and~\ref{tab:rm}.}
\label{fig:curves}
\end{figure}

\subsection{Cost of using comparisons only}
\label{sec:exp-equal}
\begin{table}[h]
\centering
\caption{Cost of comparisons only. Top: per-step comparisons vs.\ GRPO with scalar returns at the same granularity (final return, 10 seeds); equivalence by TOST with a margin of 10\% of GRPO's improvement over a random policy. Bottom: after long training, per-step comparisons vs.\ PPO with scalar rewards and a critic; non-inferiority margin 10\%. All tests pre-registered, Holm-corrected.}
\label{tab:cost}
\small
\setlength{\tabcolsep}{4pt}
\adjustbox{max width=\linewidth}{\begin{tabular}{lccl}
\toprule
 & per-step comparisons & GRPO (scalar) & verdict \\
\midrule
HalfCheetah 3M & $1481\pm162$ & $1622\pm208$ & inconclusive \\
Hopper 3M & $1881\pm414$ & $1914\pm417$ & inconclusive \\
LunarLander 2M & $262\pm4$ & $252\pm13$ & equivalent \\
Acrobot 0.5M & $-128\pm94$ & $-129\pm93$ & inconclusive \\
\midrule
 & per-step comparisons & PPO (scalar + critic) & verdict \\
\midrule
LunarLander 3M & $262\pm4$ & $235\pm18$ & non-inferior, significantly better \\
Hopper 5M & $2057\pm393$ & $1978\pm386$ & non-inferiority not established \\
HalfCheetah 3M & $1481\pm162$ & $4488\pm1026$ & worse \\
\bottomrule
\end{tabular}
}
\end{table}
No difference between per-step comparisons and GRPO is significant on any environment, but equivalence within 10\% is established only on LunarLander; on the other three the seed-to-seed variance is too large to conclude either way. After long training, comparisons are non-inferior to and significantly better than PPO on LunarLander ($p=0.009$); on Hopper the difference is not significant and non-inferiority is not established. On HalfCheetah PPO is about three times better.

\subsection{Per-step credit from a comparison-trained critic}
\label{sec:exp-critic}
\begin{table}[h]
\centering
\caption{Final return, 3M steps, 10 seeds. PPO and GiGPO use scalar rewards; the remaining four columns use comparisons only. Pre-registered: the critic variant beats the critic-free one on HalfCheetah ($p=0.013$; Hopper not significant); the model-only and doubly robust variants do not beat it.}
\label{tab:critic}
\small
\setlength{\tabcolsep}{4pt}
\adjustbox{max width=\linewidth}{\begin{tabular}{lcccccc}
\toprule
 & \multicolumn{2}{c}{scalar reward} & \multicolumn{4}{c}{comparisons only} \\
\cmidrule(lr){2-3}\cmidrule(lr){4-7}
 & PPO & GiGPO & per-step & +critic (ours) & model-only $Q$ & doubly robust \\
\midrule
HalfCheetah & $4488\pm1026$ & $334\pm138$ & $1481\pm162$ & $2196\pm523$ & $798\pm193$ & $1766\pm359$ \\
Hopper & $1761\pm302$ & $1719\pm389$ & $1881\pm414$ & $2114\pm171$ & $522\pm226$ & $2189\pm427$ \\
\bottomrule
\end{tabular}
}
\end{table}
The critic variant is significantly better than the critic-free update on HalfCheetah ($p=0.013$, AUC $p=0.006$) and has the smallest seed-to-seed spread on Hopper, where the difference is not significant. Letting the learned model supply the advantage directly ($Q(s,a)-\E_{a'}Q(s,a')$, ``model-only'') learns fastest for about 1.2M steps and then degrades steadily while the critic's pairwise accuracy stays constant at about 0.67: the policy exploits errors of the critic along the action dimension, which pairwise accuracy between states does not measure. A doubly robust combination that keeps the observed comparison as a correction does not degrade but is not better than the state-ranker critic; an ablation attributes this to the critic type (a state value obtained by averaging a learned $Q$ over actions is a worse baseline than a directly learned $V$). GiGPO's anchor-state grouping finds repeated states only at the shared initial state in continuous control and reduces to trajectory-level grouping. Every comparison-only variant remains about $2\times$ below PPO on HalfCheetah.

\subsection{Comparison with learning a reward model, and with trajectory-level win rate}
\label{sec:exp-rm}
\begin{table}[h]
\centering
\caption{Final return, 10 seeds. ``Reward model + PPO'' learns $\hat r(s,a)$ from exactly the same comparisons. Pre-registered verdict: comparable; the reward-model route is significantly worse on LunarLander (Holm), where it explains only about 70\% of the comparisons.}
\label{tab:rm}
\adjustbox{max width=\linewidth}{\begin{tabular}{lccccc}
\toprule
 & per-step & +critic (ours) & reward model + PPO & SPO & RM accuracy \\
\midrule
HalfCheetah 3M & $1481\pm162$ & $2196\pm523$ & $2371\pm690$ & $1310\pm363$ & 0.99 \\
Hopper 3M & $1881\pm414$ & $2114\pm171$ & $1827\pm162$ & $2564\pm688$ & 1.00 \\
LunarLander 2M & $262\pm4$ & $258\pm12$ & $205\pm11$ & $102\pm47$ & 0.70 \\
Acrobot 0.5M & $-128\pm94$ & $-86\pm3$ & $-208\pm110$ & $-126\pm94$ & 0.80 \\
\bottomrule
\end{tabular}
}
\end{table}

\subsection{Goal-oriented sparse reward}
On MiniGrid DoorKey-5$\times$5 (reward only on reaching the goal, variable episode length), the per-step and critic variants reach 100\% success on all seeds; area under the learning curve is not measurably different from PPO, GRPO, or the reward-model route, while SPO is 22\% slower ($p_{\text{Holm}}=0.0025$). Larger maps were excluded by a pre-registered rule because PPO itself fails on some seeds there.

\begin{table}[h]
\centering
\caption{MiniGrid DoorKey-5$\times$5, 500k steps, 10 seeds. AUC is the mean return over training (primary metric).}
\label{tab:goal}
\small
\adjustbox{max width=\linewidth}{\begin{tabular}{lcccccc}
\toprule
 & PPO & GRPO & per-step & +critic (ours) & RM + PPO & SPO \\
\midrule
AUC & $0.803\pm0.039$ & $0.841\pm0.009$ & $0.830\pm0.011$ & $0.833\pm0.020$ & $0.812\pm0.032$ & $0.649\pm0.080$ \\
final success & $0.96$ & $1.00$ & $1.00$ & $1.00$ & $0.99$ & $1.00$ \\
\bottomrule
\end{tabular}
}
\end{table}

\subsection{Comparing versus scoring with an LLM judge}
\label{sec:exp-judge}
Does a real judge compare more reliably than it scores? Without training, we collected 279 checkpoints from five stages of a LunarLander policy (2{,}046 text descriptions of ``what happened from here on'', 6{,}792 pairs; raw physical quantities, engine use and the outcome, no reward terms) and asked two judges, \texttt{gemini-3.5-flash-lite} and the stronger \texttt{gemini-3.5-flash} (minimal thinking, JSON output, identical prompts), to (i) score each description from 0 to 10, (ii) compare two descriptions directly (3 pairs per checkpoint, both orders), or (iii) rank all descriptions of a checkpoint with ties allowed (two orders). The ground truth is the discounted reward-to-go used as the oracle in all other experiments; ties count one half. Pre-registered, with checkpoint-level cluster bootstrap; total cost below \$12.

\begin{table}[h]
\centering
\caption{Effective pairwise accuracy of LLM judges against the true reward-to-go (ties count $\tfrac12$). Direct comparison is evaluated on the 823 pairs it was asked about; the other formats on all 6{,}792 pairs. Last two columns: advantage of comparison over scoring with 95\% cluster-bootstrap intervals.}
\label{tab:judge}
\small
\setlength{\tabcolsep}{4pt}
\adjustbox{max width=\linewidth}{\begin{tabular}{lcccccc}
\toprule
judge & score each & rank group & compare two & score ties & rank $-$ score & compare $-$ score \\
\midrule
Flash-Lite & 0.571 & 0.643 & 0.652 & 47\% & $+7.2$ $[5.5, 9.0]$ & $+8.2$ $[5.3, 11.2]$ \\
Flash & 0.553 & 0.657 & 0.660 & 50\% & $+10.5$ $[8.5, 12.5]$ & $+9.9$ $[6.2, 13.6]$ \\
\bottomrule
\end{tabular}
}
\end{table}

\begin{figure}[t]
\centering
\includegraphics[width=0.75\linewidth]{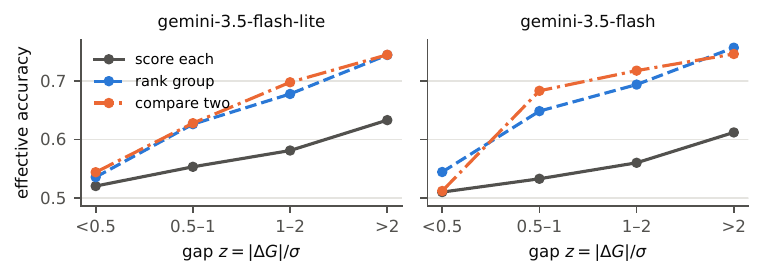}
\caption{Effective accuracy of LLM judges by the size of the true gap $z$. Near ties every format is near chance; for clear differences comparison (ranking or direct) is 10--15 points more accurate than scoring.}
\label{fig:judge}
\end{figure}

Comparison beats scoring for both judges, and the gap \emph{widens} with the stronger judge: Flash ranks more accurately than Flash-Lite ($+1.4$ points $[0.3, 2.5]$) but scores less accurately ($-1.8$ $[-3.3,-0.2]$), because its scores polarize (274 descriptions receive 0 and 731 receive 9), so half of the score-derived pairs within a checkpoint are ties, whereas its rankings almost never tie (2.7\%) and are more consistent across presentation orders (0.87 vs.\ 0.74). Near ties ($z<0.5$) every format is close to chance, which is the noise type of Proposition~\ref{prop:noise}(b).

Absolute accuracy is modest, and much of it reflects a mismatch between the stated goal and the hand-written reward rather than judge error. On pairs that differ in outcome, both judges are essentially perfect when one lander crashed (rank: 1.00), but on ``came to rest'' versus ``time limit reached'' they agree with the reward only 62--66\% of the time. In 76\% of those timeouts the lander is already standing on both legs at zero velocity; the environment withholds the landing bonus only because the engines are still firing slightly, a criterion the natural-language goal (``come to rest on the pad'') does not state. The judge's preference is reasonable under the stated goal, so the environment return understates judge accuracy --- and illustrates the motivation of this paper: a hand-written reward encodes details that are not part of the intended goal. (A first run had described these timeouts as ``still flying''; correcting the wording changed no accuracy measurably.)

\subsection{Training with an LLM judge on a requirement the reward omits}
\label{sec:exp-left}
The previous experiment measured labels only. To test training, we need a task where the hand-written reward misses part of the goal. \emph{LunarLeft} is LunarLander with one extra requirement: the lander must finish on the left half of the pad (final $x<0$). The hand-written reward is the original LunarLander reward, which ignores the requirement. The true objective subtracts 80 when a non-crash episode ends on the right. The judge (\texttt{gemini-3.5-flash-lite}) sees the goal and the requirement in natural language (Appendix~\ref{app:prompts}). Every 20 steps it ranks the descriptions of the trajectories still running, and all pairwise labels come from that ranking. Actions in the 20-step window share the checkpoint's advantage. Runs last 1M steps and cost about \$5 each. We report the last 10 updates: landing rate, crash rate, and \emph{left share}, the fraction of non-crash endings with $x<0$ (Table~\ref{tab:left}).

\begin{table}[h]
\centering
\caption{LunarLeft, last 10 updates, mean $\pm$ 95\% CI over seeds. Left share is the fraction of non-crash endings on the left half. The oracle rows were not pre-registered. The judge pilots were pre-registered with an extension rule (left share $\ge 0.70$ and landing $\ge 0.80$), which the v2 pilot fails on landing.}
\label{tab:left}
\small
\adjustbox{max width=\linewidth}{\begin{tabular}{llcccc}
\toprule
training signal & method & seeds & landed & crashed & left share \\
\midrule
hand-written reward & PPO & 10 & $0.81\pm0.08$ & $0.05\pm0.02$ & $0.52\pm0.13$ \\
true objective & PPO & 10 & $0.71\pm0.16$ & $0.05\pm0.02$ & $0.67\pm0.09$ \\
oracle comparisons & per-step, every step & 10 & $0.96\pm0.03$ & $0.02\pm0.02$ & $0.74\pm0.15$ \\
oracle comparisons & per-step, every 20 steps & 10 & $0.92\pm0.06$ & $0.02\pm0.01$ & $0.80\pm0.20$ \\
LLM judge, v1 descriptions & per-step, every 20 steps & 5 & $0.00\pm0.00$ & $1.00\pm0.00$ & --- \\
LLM judge, v2 descriptions & per-step, every 20 steps & 5 & $0.27\pm0.32$ & $0.30\pm0.07$ & $0.99\pm0.02$ \\
\bottomrule
\end{tabular}
}
\end{table}

\paragraph{Oracle comparisons.} When labels come from the true objective, per-step comparisons learn the requirement better than PPO with the hand-written reward. With checkpoints every 20 steps, left share is 0.80 versus 0.52 (Welch $p=0.018$; exploratory, not pre-registered), and landing rates are higher than PPO trained on the true objective.

\paragraph{LLM judge.} Calibration came first: before any training, the judge picked the left landing in 98\% of left-versus-right pairs. The first pilot still failed: every seed crashed in every episode for the whole run. Calibration had tested the distinction we cared about, not the comparisons the judge faces early in training, which are almost all crash versus crash. The descriptions lacked velocities and leg contact, so the judge could not tell a near-landing from a high-speed impact. Adding them (v2 descriptions) raised accuracy on clearly different crash pairs ($|\Delta G|\ge 20$) from 0.74 to 0.78. Near-tied crash pairs ($|\Delta G|<5$) stayed at chance (0.53) with both versions. With v2 descriptions, all five seeds learn the requirement (left share 0.97--1.00, versus 0.52 with the hand-written reward; Mann--Whitney $p=0.003$, exploratory). Only 27\% of episodes land, versus 92\% with oracle comparisons; the rest end in a crash (30\%) or at the time limit (43\%). Training-time agreement with the oracle labels is 0.60--0.69. The pre-registered extension rule fails on landing, so we stopped at five seeds.

The judge transmits a requirement that the hand-written reward does not contain, and it does so in every seed. It does not yet produce a policy as good as oracle comparisons do: the remaining gap tracks label accuracy on near-tied failures, which is what the policy sees most early in training.

\subsection{Label noise}
\label{sec:exp-noise}
\begin{table}[h]
\centering
\caption{LunarLander final return under label noise with a matched flip rate of about 25\%. Independent flips collapse the critic-free method; near-tie and positional noise, which leave the expected update intact (Proposition~\ref{prop:noise}), barely affect it.}
\label{tab:noise}
\adjustbox{max width=\linewidth}{\begin{tabular}{lccccc}
\toprule
 & clean & independent 25\% & near-tie & position (72\%) & unsure$\to$first \\
\midrule
per-step comparisons & $262\pm4$ & $70\pm79$ & $258\pm6$ & $250\pm12$ & $249\pm10$ \\
+ critic (ours) & $258\pm12$ & $259\pm6$ & $261\pm8$ & $263\pm6$ & $263\pm6$ \\
\bottomrule
\end{tabular}
}
\end{table}

\subsection{Comparison budget}
\label{sec:exp-budget}
\begin{table}[h]
\centering
\caption{Retention of final return relative to the full label budget (all pairs at every step), 10 seeds. ``One pair per step'': a random perfect matching at every step ($\approx15\%$ of labels). ``All pairs every 10 steps'': $\approx10\%$. Both combined: $\approx1.5\%$. Bold: significantly worse than full budget (Holm).}
\label{tab:budget}
\small
\adjustbox{max width=\linewidth}{\begin{tabular}{lccc}
\toprule
 & one pair per step ($\approx$15\%) & all pairs every 10 steps ($\approx$10\%) & both ($\approx$1.5\%) \\
\midrule
LunarLander, per-step & 0.98 & 1.02 & 0.98 \\
LunarLander, +critic & 0.98 & 1.02 & 1.00 \\
HalfCheetah, per-step & \textbf{0.81} & 0.99 & \textbf{0.83} \\
HalfCheetah, +critic & \textbf{0.69} & 0.77 & \textbf{0.64} \\
\bottomrule
\end{tabular}
}
\end{table}

\begin{figure}[t]
\centering
\includegraphics[width=0.5\linewidth]{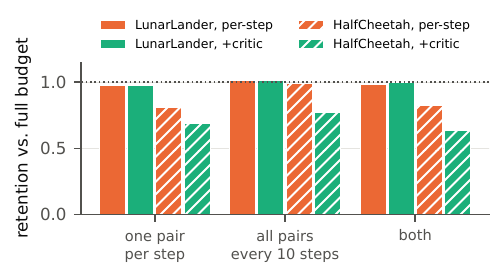}
\caption{Retention of final return under reduced comparison budgets (hatched: HalfCheetah). Spending labels on all pairs at a few time points keeps performance; one pair per step does not.}
\label{fig:budget}
\end{figure}
At a similar budget, comparing all pairs at a few time points is far better than comparing one pair at every step: a single comparison makes each advantage a noisy $\pm1$, while reward-to-go changes little within ten steps. The critic variant loses its full-budget advantage at reduced budgets, possibly because its critic receives far fewer training pairs with a fixed number of epochs (not tested).

\subsection{Sensitivity to hyperparameters}
Over a $3\times3$ grid of learning rate and update epochs with 10 seeds per setting, we measure the worst-setting retention $(\text{AUC}_{\min}-R_0)/(\text{AUC}_{\max}-R_0)$. On LunarLander, per-step comparisons ($0.89$) and GRPO ($0.86$) are far more robust than PPO ($0.54$; difference to comparisons $+0.35$, 95\% CI $[0.28, 0.41]$). On Acrobot the order reverses: PPO $0.92$, GRPO $0.83$, comparisons $0.74$ (difference $-0.31$ $[-0.54,-0.13]$). Comparisons and GRPO do not differ on either task. An earlier 3-seed sweep had suggested that comparisons are the least sensitive method; the pre-registered 10-seed test does not support that claim, and we do not make it.

\section{Practical guidance}
\label{sec:guidance}
\begin{itemize}
\item \textbf{Spend labels on all pairs at a few time points.} At about 10\% of the full budget, comparing all pairs every 10 steps keeps performance, whereas one comparison per step at every step does not (Section~\ref{sec:exp-budget}).
\item \textbf{Ask judges to compare, not to score.} With LLM judges, ranking or pairwise comparison is 7--10 points more accurate than 0--10 scores, and a listwise ranking of $K$ outcomes yields all $\binom{K}{2}$ labels in one call, whereas independent scoring needs $K$ calls.
\item \textbf{Allow ties and do not worry about near-tie noise.} Errors concentrated near ties and position bias leave the expected update intact (Proposition~\ref{prop:noise}); they cost variance, not direction.
\item \textbf{Use the comparison-trained critic when states diverge and labels are plentiful;} keep the observed comparison in the advantage. Never let a learned model supply the advantage alone (Section~\ref{sec:exp-critic}).
\item \textbf{Check whether a per-step reward explains the comparisons} before fitting a reward model: if its pairwise accuracy is well below one, compare directly instead (Section~\ref{sec:exp-rm}).
\item \textbf{Mind stochastic outcomes.} The update prefers actions that win more often, not those with higher mean (Proposition~\ref{prop:skew}); if rare large losses matter, comparisons alone will under-weight them.
\end{itemize}

\section{Limitations}
\begin{itemize}
\item Outside Section~\ref{sec:exp-left}, all labels are generated from environment returns. These experiments show that the algorithms do not need reward magnitudes, not that they learn where no reward can be written down. Training with a real LLM judge is a five-seed pilot on one task; it transfers the stated requirement but lands far less often than oracle comparisons.
\item Groups require $K$ rollouts from a shared initial state, i.e.\ a resettable simulator.
\item Under stochastic rewards the objective is a win rate, not expected return (Proposition~\ref{prop:skew}).
\item On HalfCheetah every comparison-only method remains about $2\times$ below PPO with scalar rewards.
\item Small networks and classic control tasks; harder goal-oriented tasks are limited by exploration, where comparisons among all-failing groups carry no signal.
\end{itemize}

\section*{Reproducibility statement}
All experiments run on a single CPU machine; every number in the tables is regenerated from the raw per-seed logs by one script (\texttt{paper\_assets.py}) with a single definition of each metric. Each main claim has a pre-registration file written before its runs, listing the test, the correction, the decision rule and our predictions; Appendix~\ref{app:prereg} summarizes them, including the predictions that turned out wrong, and Appendix~\ref{app:deviations} lists every deviation. Judge prompts are given verbatim in Appendix~\ref{app:prompts}; all judge requests and responses are cached and released.

\bibliographystyle{plainnat}
\bibliography{refs}

\clearpage
\appendix
\section{Implementation details}
\label{app:impl}

\begin{table}[h]
\centering
\caption{Hyperparameters shared by all methods unless stated. No method was tuned on evaluation seeds.}
\label{tab:hparams}
\small
\begin{tabular}{p{3.6cm}p{11.6cm}}
\toprule
actor & MLP $128$--$128$, tanh; categorical (discrete) or Gaussian with state-independent log-std (continuous) \\
observation normalization & running mean/std, clipped to $[-10,10]$, shared by all methods \\
group size $K$ / batch & 8 trajectories per group from a shared reset seed; groups until $\ge 8192$ steps per update \\
PPO-clip & $\epsilon=0.2$, lr $3\cdot10^{-4}$ (Adam), 10 epochs, minibatch 512, gradient-norm clip 0.5 \\
discount & $\gamma=0.99$ (labels compare discounted reward-to-go) \\
advantage scaling & comparison/GRPO advantages divided by batch std (not re-centered); PPO: batch z-score \\
PPO critic & MLP $128$--$128$, GAE $\lambda=0.95$, separate Adam optimizer \\
comparison critic & MLP $128$--$128$ on observations; Bradley--Terry loss, lr $10^{-3}$, 4 epochs, $\le 4\cdot10^4$ pairs per update, minibatch 1024; trained after the advantages of the batch are computed \\
model-only $Q$ / doubly robust & $Q$-network on $(s,a)$, same training; $M=8$ sampled actions per state \\
reward model + PPO & $\hat r(s,a)$ MLP $128$--$128$; Bradley--Terry on discounted sums of $\hat r$; buffer of the last 5 updates, 2 epochs per update; $\hat r$ scaled by batch std \\
SPO & queue of the last 64 trajectories; score in $[-1,1]$ split evenly over the trajectory's steps; PPO with critic (the original uses SAC for continuous control) \\
GiGPO & episode-level z-score $+\,1.0\times$ z-score within groups of identical states \\
\bottomrule
\end{tabular}
\end{table}

\paragraph{Environments and budgets.} Gymnasium CartPole-v1 (0.4M steps), Acrobot-v1 (0.5M), LunarLander-v3 (2M; 3M for long runs), Hopper-v4 and HalfCheetah-v4 (3M; 5M and 12M for long runs), MiniGrid DoorKey-5$\times$5 (0.5M; flattened $7\times7\times3$ egocentric view). ``Final'' is the mean logged return over the last 10\% of the nominal step budget; AUC is the mean of 50 equal-width step bins. Seeds: evaluation seeds 1--10 (0--9 for long runs), tuning seeds 100--109, pilot seeds 0 and 100--101; pilot runs never enter evaluation.

\paragraph{Comparison budget.} ``One pair per step'' draws a random perfect matching of the running trajectories at every step; ``every 10 steps'' compares all pairs at $t\in\{0,10,20,\dots\}$ and lets the actions of each window share the advantage of its start.

\section{Pre-registered experiments}
\label{app:prereg}
Each row had a file with tests, corrections, decision rule and predictions written before its runs.

\begin{table}[h]
\centering
\small
\begin{tabular}{p{3.4cm}p{5.8cm}p{5.2cm}}
\toprule
experiment & decision & predictions that were wrong \\
\midrule
comparison critic (Sec.~\ref{sec:exp-critic}) & critic beats critic-free on HalfCheetah ($p=0.013$); model-only and doubly robust variants do not help & doubly robust would help \\
reward model and SPO (Sec.~\ref{sec:exp-rm}) & comparable; reward model significantly worse on LunarLander & reward model similar to comparisons on LunarLander and Acrobot; SPO strong on Hopper not anticipated \\
goal-oriented task & comparisons usable (100\% success); SPO 22\% slower & comparisons 10--15\% slower than PPO; SPO competitive \\
label noise, independent & critic significantly more robust at 25\% flips & --- \\
label noise, correlated & does not generalize (0/6 significant) & critic advantage persists under position bias \\
comparison budget & LunarLander robust at all budgets; HalfCheetah only for all-pairs-every-10-steps & critic more label-efficient; every-10-steps hurts HalfCheetah \\
ties and local optima (15 seeds) & ties do not change the escape rate & ties help escape \\
10-seed confirmation & equivalence to GRPO only on LunarLander; non-inferiority to PPO only on LunarLander; ``least hyperparameter-sensitive'' not supported & equivalence on Acrobot; comparisons more robust than PPO on both tasks \\
LLM judge v1/v2 & comparison more accurate than scoring for both judges & large-gap accuracy near 1; wording fix helps; stronger judge narrows the gap \\
LLM-judge training, LunarLeft (Sec.~\ref{sec:exp-left}) & v1 calibration passed, v1 pilot failed (all crashes); v2 calibration passed; v2 pilot fails the extension rule on landing, stopped at 5 seeds & left share 0.70--0.75 (observed 0.99); landing $\ge 0.85$ (observed 0.27) \\
\bottomrule
\end{tabular}
\end{table}

\section{Deviations, bugs and reruns}
\label{app:deviations}
We list every issue found during the project; none affects a reported number after the stated fix.
\begin{enumerate}
\item An early version stored observations as views into a reused array, so stored states were overwritten; found in a sanity check before any reported run.
\item While adding the comparison critic, an edit displaced PPO's return collection; caught in a pilot, fixed, and PPO's output was verified to be line-for-line identical to earlier runs before all later experiments.
\item The first SPO implementation gave the score only at the final step and failed on CartPole; it was corrected to the published algorithm (score split over steps, queue of recent trajectories) before the pre-registered run.
\item A noise-scale fallback produced near-zero noise on Acrobot when all first-batch returns were equal; the affected noise runs were discarded and rerun with a corrected scale.
\item The first LLM-judge descriptions called time-limit episodes ``still flying'' even when the lander stood on its legs; the corrected run changed no accuracy measurably and is the one reported.
\item A pilot judge run overwrote the stored judgments of the first full run; they were rebuilt from the request cache without new calls and the analysis output was verified identical.
\item The first LunarLeft judge pilot failed because its calibration did not cover crash-versus-crash pairs; we added a criterion for them and new descriptions in a second pre-registration (v2) before running again. Both pilots are reported.
\item The two-seed LunarLeft mini-pilot reached its \$3 spending cap at about 615k of the planned 650k steps, after the continue rule had already been met at 603k; the full pilot reused its cached judge calls.
\item The stronger judge reached its spending cap with 24 of 300 repeat scorings missing (a secondary metric only).
\item Earlier internal analyses used windows defined by the longest run instead of the nominal budget; all reported tables are regenerated with the nominal budget (differences below 0.2\%).
\end{enumerate}

\section{Judge prompts}
\label{app:prompts}
All three formats share this instruction, followed by the format-specific request. Output is constrained to JSON by a schema; thinking level minimal.
\begin{verbatim}
You are judging the behaviour of a lunar lander in a simulator.
Goal: land gently and upright on the landing pad, which is on flat ground
between two flags at horizontal position x close to 0 (height 0 is the
ground). Use little engine fuel. Do not crash or fly out of bounds. Coming
to rest on the pad is the ideal outcome.
Each description covers what happened from one moment until the end of the
episode. Judge only how well that stretch achieves the goal.
\end{verbatim}
\begin{itemize}
\item Score: ``Description: \dots\ Rate how well this achieves the goal from 0 (very bad) to 10 (ideal).''
\item Compare: ``Description A: \dots\ Description B: \dots\ Which description better achieves the goal? Answer A, B, or SAME if they are about equally good.''
\item Rank: ``Description A: \dots\ \dots\ Rank all $n$ descriptions from best to worst. Put descriptions that are about equally good in the same tier. Every letter must appear exactly once.''
\end{itemize}
For LunarLeft (Section~\ref{sec:exp-left}) the instruction adds: ``Additional requirement: the lander must finish on the LEFT half of the pad (final x below 0). Finishing on the right half (final x above 0) fails this requirement even if the landing is otherwise good.'' Training uses the ranking format only. The v2 descriptions give, for up to five snapshots, $(x,\text{height},v_x,v_y,\text{tilt},\text{legs touching})$, plus velocity, tilt and leg contact just before the end.

An example description (first lines):
\begin{verbatim}
Lasts 400 more steps, then: the time limit was reached.
Main engine fired on 56 steps, side engines on 128 steps.
Snapshots:
  +0 steps: x=-0.14, height=0.14, vx=-0.07, vy=-0.04, tilt=-3 deg, ...
  +25 steps: x=-0.14, height=0.11, vx=+0.02, vy=-0.09, tilt=-2 deg, ...
\end{verbatim}

\section{Additional results}
\label{app:results}

\begin{table}[h]
\centering
\caption{Bandit diagnostics (5 arms, groups of 8, 20 seeds): number of seeds whose final policy picks the arm with the highest mean. In the skewed task one arm wins 90\% of comparisons but has a lower mean.}
\label{tab:bandit}
\small
\adjustbox{max width=\linewidth}{\begin{tabular}{lccccc}
\toprule
 & GRPO (z-score) & linear & comparisons & comparisons, 30\% flips & tanh \\
\midrule
Gaussian noise & 20/20 & 20/20 & 20/20 & 20/20 & 20/20 \\
heavy-tailed noise & 20/20 & 18/20 & 20/20 & 20/20 & 20/20 \\
reward $\times1000$ & 20/20 & 14/20 & 20/20 & 20/20 & 20/20 \\
skewed (win often, lose big) & 0/20 & 20/20 & 0/20 & 0/20 & 0/20 \\
\bottomrule
\end{tabular}
}
\end{table}

\begin{table}[h]
\centering
\caption{Same-granularity comparison with and without per-step credit (5 seeds, seeds 0--4). Trajectory-level variants credit every step with the whole-trajectory signal.}
\label{tab:phaseb}
\small
\setlength{\tabcolsep}{3pt}
\adjustbox{max width=\linewidth}{\begin{tabular}{lccccc}
\toprule
 & PPO (scalar) & GRPO traj.\ (scalar) & comparisons traj. & GRPO per-step (scalar) & comparisons per-step \\
\midrule
CartPole 0.4M & $500\pm0$ & $500\pm1$ & $500\pm0$ & $499\pm2$ & $499\pm2$ \\
Acrobot 0.5M & $-82\pm1$ & $-90\pm3$ & $-87\pm2$ & $-86\pm2$ & $-86\pm1$ \\
LunarLander 2M & $223\pm42$ & $228\pm67$ & $151\pm149$ & $242\pm33$ & $262\pm7$ \\
Hopper 2M & $2052\pm663$ & $1393\pm343$ & $1627\pm552$ & $1287\pm274$ & $1594\pm451$ \\
HalfCheetah 2M & $2718\pm1558$ & $2\pm141$ & $-39\pm64$ & $1123\pm170$ & $1094\pm179$ \\
\bottomrule
\end{tabular}
}
\end{table}

\begin{table}[h]
\centering
\caption{Hyperparameter grid: mean AUC over 10 seeds per setting (columns: learning rate, then update epochs).}
\label{tab:sweep}
\small
\setlength{\tabcolsep}{3pt}
\adjustbox{max width=\linewidth}{\begin{tabular}{llccccccccc}
\toprule
 & & \multicolumn{3}{c}{lr $10^{-4}$} & \multicolumn{3}{c}{lr $3\cdot10^{-4}$} & \multicolumn{3}{c}{lr $10^{-3}$} \\
\cmidrule(lr){3-5}\cmidrule(lr){6-8}\cmidrule(lr){9-11}
env & method & 4 & 10 & 20 & 4 & 10 & 20 & 4 & 10 & 20 \\
\midrule
LunarLander & PPO (scalar) & -1 & 61 & 104 & 57 & 114 & 147 & 117 & 137 & 149 \\
LunarLander & GRPO (scalar) & 139 & 153 & 144 & 143 & 139 & 136 & 115 & 132 & 108 \\
LunarLander & per-step comparisons & 146 & 155 & 150 & 154 & 155 & 147 & 134 & 148 & 118 \\
Acrobot & PPO (scalar) & -190 & -181 & -181 & -189 & -177 & -163 & -180 & -161 & -166 \\
Acrobot & GRPO (scalar) & -229 & -261 & -226 & -231 & -210 & -227 & -217 & -214 & -240 \\
Acrobot & per-step comparisons & -231 & -204 & -237 & -233 & -207 & -237 & -229 & -242 & -280 \\
\bottomrule
\end{tabular}
}
\end{table}

\begin{table}[h]
\centering
\caption{Acrobot under label noise (final return, 10 seeds; in parentheses the number of seeds ending below $-400$).}
\label{tab:noise-acrobot}
\small
\setlength{\tabcolsep}{3pt}
\adjustbox{max width=\linewidth}{\begin{tabular}{lccccc}
\toprule
 & clean & independent 25\% & near-tie & position (72\%) & unsure$\to$first \\
\midrule
per-step comparisons & $-128\pm94$ (1) & $-316\pm141$ (5) & $-88\pm2$ (0) & $-107\pm26$ (0) & $-91\pm4$ (0) \\
+ critic (ours) & $-86\pm3$ (0) & $-101\pm21$ (0) & $-89\pm3$ (0) & $-142\pm93$ (1) & $-177\pm123$ (2) \\
\bottomrule
\end{tabular}
}
\end{table}

\paragraph{Numerical checks of the theory.} Theorem~\ref{thm:multistep}: exact enumeration versus $4\times10^5$ sampled groups on five random MDPs with stochastic transitions, rewards and termination (all components within 2.02 standard errors; omitting $c_t$ is off by 14--124). Proposition~\ref{prop:detmdp}: 300 random deterministic MDPs with tied integer rewards, no failures in 1{,}670 / 2{,}176 / 277 cases. Proposition~\ref{prop:det}: a random search over $2\times10^5$ reward/policy pairs found that the mass on maximizers is not monotone in 3.2\% of cases, which is why the proposition states stability rather than monotone improvement.

\end{document}